\documentclass[conference]{IEEEtran}
\IEEEoverridecommandlockouts
\usepackage{cite}
\usepackage{amsmath,amssymb,amsfonts}
\usepackage{algorithmic}
\usepackage{graphicx}
\usepackage{textcomp}
\usepackage{xcolor}
\usepackage{hyperref}
\usepackage[]{todonotes}
\usepackage{soul}
\usepackage{cite}
\usepackage{balance}
\usepackage{textcomp}
\ifCLASSINFOpdf
\else
\usepackage[dvips]{graphicx}
\fi
\usepackage{algorithm}      
\usepackage{algorithmic}

\usepackage{varwidth}% http://ctan.org/pkg/varwidth

\usepackage{amsthm}

\newtheorem{prp}{Proposition}

\theoremstyle{definition}
\newtheorem{rem}{Remark}

\newcommand{\R}{\mathbf{R}}
\newcommand{\C}{\mathbf{C}}
\newcommand{\N}{\mathbf{N}}
\newcommand{\E}{\mathbb{E}}
\newcommand{\ones}{\mathbf{1}}
\newcommand{\norm}[1]{\lVert #1 \rVert}
\newcommand{\kronecker}{\otimes}

\newcommand{\Prob}{\Pr}
\begin{document}

% \section{Local \hmath $q$-adic refinement in \hmath $n$D}

\title{Finite State Markov Modeling of C-V2X Erasure Links For Performance and Stability Analysis of Platooning Applications}
\author{Mahdi Razzaghpour$^{1}$, Adwait Datar$^{2}$, Daniel Schneider$^{3}$, Mahdi Zaman$^{1}$,\\ Herbert Werner$^{2}$, Hannes Frey$^{3}$, Javad Mohammadpour Velni$^{4}$, Yaser P. Fallah$^{1}$

\thanks{$^{1}$ Connected \& Autonomous Vehicle Research Lab (CAVREL), University of Central Florida, Orlando, FL, USA. \tt\small {razzaghpour.mahdi@knights.ucf.edu}}%
\thanks{$^{2}$ Department of Electrical and Computer Engineering, Hamburg University of Technology, Germany.}
\thanks{$^{3}$ Department of Electrical and Computer Engineering,  University of Koblenz, Germany.}
\thanks{$^{4}$ School of Electrical and Computer Engineering, University of Georgia, Athens, GA, USA.}
\thanks{The code and supplementary material are available at: \url{www.github.com/MahdiRazzaghpour/MJLSPlatooning}}
}

\maketitle

\begin{abstract}
Cooperative driving systems, such as platooning, rely on communication and information exchange to create situational awareness for each agent. Design and performance of control components are therefore tightly coupled with communication component performance. The information flow between vehicles can significantly affect the dynamics of a platoon. Therefore, both the performance and the stability of a platoon depend not only on the vehicle's controller but also on the information flow Topology (IFT). The IFT can cause limitations for certain platoon properties, i.e., stability and scalability. Cellular Vehicle-To-Everything (C-V2X) has emerged as one of the main communication technologies to support connected and automated vehicle applications. As a result of packet loss, wireless channels create random link interruption and changes in network topologies. In this paper, we model the communication links between vehicles with a first-order Markov model to capture the prevalent time correlations for each link. These models enable performance evaluation through better approximation of communication links during system design stages. Our approach is to use data from experiments to model the Inter-Packet Gap (IPG) using Markov chains and derive transition probability matrices for consecutive IPG states. Training data is collected from high fidelity simulations using models derived based on empirical data for a variety of different vehicle densities and communication rates. Utilizing the IPG models, we analyze the mean-square stability of a platoon of vehicles with the standard consensus protocol tuned for ideal communication and compare the degradation in performance for different scenarios. We additionally present some initial theoretical results that shed some light on the connection between the independent identical distributed (i.i.d.) modeling approach which neglects the time correlations in links and the proposed Markovian approach.
\end{abstract}


%%%%%%%%%%%%%%%%%%%%%%
\section{Introduction}
%%%%%%%%%%%%%%%%%%%%%%
\noindent Cooperative driving, enabled by adding communication to single Automated Vehicles (AVs), contributes to safety, and efficiency\cite{D2cav}. Two applications of Connected Automated Vehicles (CAVs) are Cooperative Adaptive Cruise Control (CACC) and Platooning. Both applications are designed to enhance driving comfort and safety by relieving the driver of the need to adapt to a cruising speed to match a group of vehicles on a freeway. The objectives of platooning are minimizing the expected value of spacing (and spacing error) between vehicles and keeping the desired distance, smoothing the engine/brake input to the platoon system, and keeping the velocity and acceleration in a reasonable and comfortable range. The main goal is to ensure that all platoon members move at the same speed while maintaining the desired formation geometry, which is specified by a desired intervehicle gap policy. A gap policy dictates the desired following spacing between vehicles. A constant spacing policy, which is called Platooning, makes vehicles maintain a constant distance from their immediate predecessor. A constant time headway gap, which is called CACC, dictates that the desired following distance be proportional to the speed of the vehicle, the higher the speed is, the longer the distance. Platooning is generally more sensitive than CACC to communication losses, mainly due to its very close coupling between vehicles\cite{RZ_impact}. The information exchange flow and Information Flow Topologies (IFT) define how the vehicles in a platoon exchange information with each other. However, changes in wireless communication attributes may cause loss of packets, which can introduce uncertainties in the vehicular networked systems and significantly impact control performance.

\begin{figure}[t]
\centering
\includegraphics[width=0.8\columnwidth]{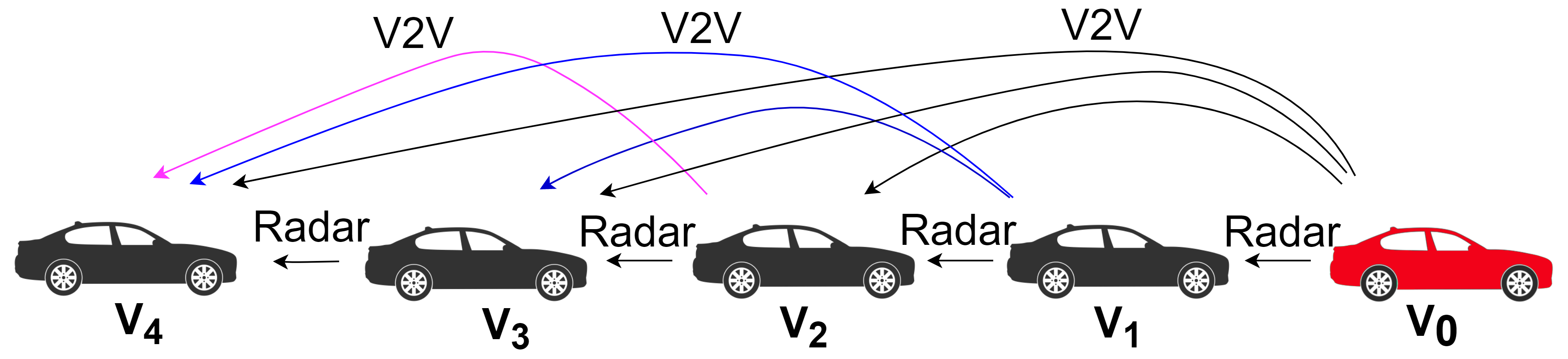}
\caption{Connectivity between members of a platoon. In addition to V2V communication, every vehicle has access to information of the preceding one via radar.}
\label{fig:Connectivity}
\end{figure}

A typical platoon considered here is shown in Figure \ref{fig:Connectivity}. Each vehicle in a platoon has a unique ID, i.e., $i = 0, 1, 2, \dots,N-1$, where i = 0 is the ID of the leader vehicle and the following vehicles have IDs $i \in\{1,2,\dots,N-1\}$ as in Figure \ref{fig:Connectivity}, where $N=5$. At each time slot t, all N vehicles broadcast their information to their following vehicles. The information could include precise speed information, acceleration, fault warnings, and warnings of forward hazards. In a vehicle platoon, the position, velocity, and acceleration have high importance. By using both sensor and communication information, the vehicle string has a quicker response to any change in road traffic. 

Studies have shown that CACC has the potential to substantially increase highway capacity when it reaches a moderate to high market penetration \cite{Impacts_CACC}. Bart \textit{et al.} concluded that CACC can potentially double the capacity of a highway lane at high-CACC market penetration \cite{CACC_Char}. Ramyar \textit{et al.} studied the impact of communication uncertainties on the string stability of a platoon. The entire platoon was formulated in the networked control system framework. Moreover, the stochastic packet loss is assumed as a Markov chain, and its effect on changing the platoon topology is modeled as a Markov jump linear system \cite{Unreliable_Communication}. In this paper, we study the impact of erasure links on the platoon with two main differences from the mentioned paper. First, we use realistic data for C-V2X technology, and second, we used All-predecessor-leader following (APLF) topology motivated by \cite{Stability_Scalability}. In this paper, Zheng \textit{et al.} studied the influence of IFT on the internal stability and scalability of homogeneous vehicular platoons. Segata \textit{et al.} provided a first evaluation of the impact of the scheduling algorithm on platoon control \cite{CV2X_Platooning}. Eben \textit{et al.} studied the internal stability and controller synthesis for vehicular platoons with generic communication topologies with a linear vehicle model \cite{Complex_Eigenvalues}.

In most state-of-the-art methods, packet loss is considered as an independent identically distributed  (i.i.d.) random process. But the temporal relationship between consecutive receptions has not been studied. In case of adverse channel status, if a packet is lost at any instance, losing the next packet is probable to some degree depending on the specific conditions. While assuming packet loss an i.i.d. process is legitimate, the temporal relationship can not be ignored because the communication link is a dynamic system and a random process that changes over time in a correlated manner\cite{Communication_Aspects}. Tractable mathematical models are required to accurately describe time variations of the link status and its dynamics. The main goal of this paper is to capture these time dependencies for C-V2X technology and utilize them for platoon performance analysis. We modeled communication links in terms of data reception status in the last timestep. At any given instance, the data is either received or lost. If the data is lost, the communication link is considered disconnected at that instance. We model the channel interruptions from link erasure and the resulting switching network topologies by Markov chains and derive their probability transition matrices from channels’ stochastic characterizations. The main contributions of this paper are as follows:
\begin{itemize}
\item We use Markov chain models to represent erasure links and the resulting switching network topologies. Transition probability matrices are derived from high fidelity simulation and empirical data to represent the network dynamics.
\item We study the impact of communication erasure links on platoon control performance for a platooning example with a standard consensus protocol tuned for ideal communication.
\item We present some early-stage theoretical results which bring insights on the relation between the i.i.d. modeling approach and the proposed Markovian approach.
\end{itemize}
%%%%%%%%%%%%%%%%%%%%%%
%\section{Related Work}
%%%%%%%%%%%%%%%%%%%%%%
%%%%%%%%%%%%%%%%%%%%%%%%%%%%%%%%
\section{Communication Modeling}
%%%%%%%%%%%%%%%%%%%%%%%%%%%%%%%%
\noindent C-V2X is the latest and most competitive contender of the ITS spectrum at this moment. It is expected to operate on 20MHz ITS band in the 5.9 GHz spectrum. This spectrum is dedicated towards vehicular operations by FCC \cite{fccNov2020} to provide vehicular safety and infotainment services. Hence this technology has been tailored with a specific goal of ensuring safety by accommodating the critical requirements posed in a vehicular environment. Due to the environment being subject to high-speed and abrupt change of channel condition, C-V2X has specific latency and periodicity requirements, and resource-access algorithms are built with keeping those minimums checked at all times irrespective of channel condition, speed, or traffic density. In this section, we will primarily focus on particular aspects of C-V2X which are relevant in the context of platooning and thus were considered in the modeling procedure.

C-V2X is developed on top of the Device-to-Device (D2D) communication stack released by 3GPP rel. 12 and upgraded for vehicular applications with modifications under rel. 14. The D2D communication enables C-V2X to operate both within the coverage of a base station (named as mode-3) via Uu interface and out of central coverage (named as mode-4) using PC5 (also known as sidelink) interface. This paper focuses on the mode-4 operation, where direct communication can be established among neighboring C-V2X equipped users via periodical exchange of basic safety messages (BSM). These messages are 180-300 bytes in general, with room for up to 1400 bytes for special purpose service messages. BSM contains spatial and dynamic status information about the user vehicle and is broadcast for most applications. Successful reception of BSM allows the neighboring receivers to make well-informed decisions and plan to minimize safety threats. The required periodicity is a minimum of 100ms, i.e. the vehicles are required to post updates via BSM at least once every 100ms and thus the neighboring vehicles can expect these updates to accurately track the neighboring User Equipments (UEs). As with any wireless network with a huge number of independent users, packet collision may occur due to interference or non-line-of-sight (NLOS) link between communicating pairs. Subsequent efforts that have been tested to mitigate these negative remarks include a congestion control mechanism that adjusts the periodicity by preemptively sensing the channel busy status \cite{toghi2019analysis}. C-V2X inherits this mechanism from its predecessor, DSRC, a wifi-based V2X protocol, which varies the BSM periodicity from 100ms up to 600ms depending on the neighborhood traffic density.

From C-V2X operational perspective, the air interface can be thought of as a grid-like pool of resources, divided into 1ms time-granularity and 12 frequency blocks each of 15kHz bandwidth. For BSM transmission, the frequency spectrum can be utilized in terms of subchannels i.e. a collection of blocks, the width of which can vary depending on the message size. A successful BSM transmission consumes multiple subchannels which the UEs in the vicinity can decode by sensing the power spectral density. Besides providing updated information, this structure also facilitates an UE's own transmission, particularly in distributed mode of operation (mode 4). Using a reception history, an UE can perform an autonomous resource selection mechanism standardized in 3GPP as semi-persistent scheduling. Upon reception of BSMs, an UE tracks the resource coordinates (in time and frequency) utilized by other UEs in the past 1-second history. Resources within this window are ranked in ascending order of sensed Reference Signal Received Power (RSRP). Untracked resources due to half-duplex are excluded from this list. A free resource or one that is used by a far UE will show a low RSRP and be ranked before a resource used by near UE. The lower a resource ranks, the lower the possibility of collision from using that resource. So among the least used 20\% resource blocks, one resource is randomly selected for the UE's next few transmissions. The number of consecutive transmissions with the same resource is subject to random selection within the range of [2 6]. An UE can also make a one-hop transmission using a random resource from the subset of available resources to break persistent collision with a particular user \cite{bazzi1shot}. This hop occurs at random intervals and the resource usage immediately returns to an SPS-based allocation from the consequent transmission.

For platooning applications, the subject vehicle group can be considered in close vicinity. So the implication of the resource selection pattern of neighboring UEs is of our particular interest in this work. Since all C-V2X equipped UEs use similar sensing-based resource selection, the subgroup of platooning UEs tend to receive packets with similar RSRP and thus generates a similarly ranked list of resources for their transmission. This intrinsic system feature was considered in our modeling. This pattern in resource usage is affected by distance from receivers and message dissemination rate as a function of traffic density.

Most works on vehicle platoon controller design and performance analysis for general IFT rely on the assumption that the communication environment is perfect. We model the communication link with a first-order Markov model which captures the time correlations in time prevalent in vehicular networks. The Inter-Packet Gap (IPG) metric is chosen to assess the system performance in terms of both network reliability and latency. IPG is defined as the interval between two consecutive receptions from a particular UE. To successfully track a neighboring UE, it is important to have an established link with that UE. Thus, IPG is a key performance indicator of how well a car can be tracked by an ego vehicle. Ideally, IPG is desired to be as small as possible and feasible with the density consideration, floored to 100ms as per standard. We model IPG with Markov Chain for communication performance evaluation. Each vehicle $i$, will use the latest received information about its neighboring vehicles in making dynamic decisions of its local control. If information about a vehicle is older than one timestep, subject vehicle $i$ will use interpolated status from the last received packet from that vehicle. Information older than 1000ms (=1sec) is considered obsolete and that vehicle is assumed out of the relevant network region.

%%%%%%%%%%%%%%%%%%%%%%%%%%%%%%%%%%%%%%%%%%
\subsection{Markov Chain Modeling for IPG}
%%%%%%%%%%%%%%%%%%%%%%%%%%%%%%%%%%%%%%%%%%
\noindent In C-V2X, the transmission opportunities are lower-bounded by 100ms time intervals. In other words, IPG fluctuates between the closest multiple of 100ms. In a low-density situation, vehicles are capable of maintaining 100ms IPG, however, an increase in the vehicle density spreads the distribution of the IPG as much as over 1s gap between two consecutive packet receptions for neighboring pairs. In our modeling, data points with IPG higher than 1000 ms were omitted. To emphasize, in reality, C-V2X devices do not have periodicities that are almost equal to a multiple of 100ms, so resources will overlap less frequently, leading to reduced IPG.

It has been established together with experimental verifications that Wireless channels can be reasonably represented by Markov chain models. In particular, Markovian models are shown to be useful representations for binary channels, radio communication systems, and fading channels\cite{Markov_modeling,Erasure_Channels}.

For a given communication link, packet loss is a random process. For simplicity, it is commonly assumed that the process is i.i.d.. However, in a vehicular network, packet loss probability at the current timestep depends on link conditions of the recent past due to environmental factors or communication system management. Typically environment conditions are temporally related since an underlying cause for a link interruption (such as an obstacle in a communication antenna’s line of sight or by two vehicles using the same resource) creates a dynamic relationship in consecutive transmissions. This temporal dependence is modeled in this paper by a Markov chain so that a presently connected link can have different packet loss properties from the case when this link is broken. For platooning application purposes, the vehicle considers the immediate vehicle in lane (front vehicle) and uses the information received from all vehicles in the range of communication as in Figure \ref{fig:Connectivity}.

Strictly speaking, the state-transition probabilities $\Prob(IPG_{m+1}|IPG_{m})$ should be derived from integrating the joint pdf of the multi-variate random variable over two consecutive IPG indices and over the desired regions. However, such a joint pdf is not available in closed form. Instead, one can perform Monte Carlo techniques to find the empirical estimate of probability $\Prob(IPG_{m+1}|IPG_{m})$. This method enables a better approximation of communication links in real-time applications.

A valid and realistic physical layer realization is of great importance in a vehicular communication simulation. We used a high-fidelity link-level event-based ns-3 network simulator \cite{rouil2016long} to simulate the IPG values and to implement the distributed congestion control algorithm specified in SAE J2945/1 standard on top of the C-V2X stack \cite{toghi2019analysis,Performance_CV2X,CV2X_CC}. The Markov chain model is used to represent temporally dependent link properties, link transmission scheduling strategies, co-channel interference avoidance methods, among others.

Based on factors like distance, communication rate, and traffic density, we now specify the Markov chain $\Theta_{n}$ for IPG values. $\Theta_{n}$ is a Markov chain taking values in a finite set $IPG \in \{100 ms, 200 ms, \dots, 1000 ms\}$, and $\Theta_{n}$ is a 10 by 10 transition probability matrix. Also, the effect of the initial state fades out for a long sequence. Therefore, one can generate data according to its steady-state distribution. Under the first-order Markovian assumption, we need the probability of the next packet reception given the time of previous packet reception, or in other words, $\Prob(IPG_{m+1}|IPG_{m})$. A first-order Markovian assumption means that given the previous IPG value, the current value is statistically independent of all other past and future IPG values.

%%%%%%%%%%%%%%%%%%%%%%%%%%%%%%%%%%%%%%%%%%%%%%%%%%%
\subsection{Markov Chain Modeling for Connectivity}
%%%%%%%%%%%%%%%%%%%%%%%%%%%%%%%%%%%%%%%%%%%%%%%%%%%
\noindent The network interconnection can be represented by an upper triangular random topology matrix ($J$) whose element $(i,j)$ indicates whether vehicle $j$ is connected to vehicle $i$ and can obtain the information from it or not. The components of $J$ can be either $0$ or $1$, depending on accessing to the information of preceding vehicles ($J_{i,j} = 1$) or not ($J_{i,j} = 0$). Information access can be either be via a communication link or radar information. If we consider two consecutive vehicles, we always have access to radar information. If we consider far vehicles, let's say vehicles with large distances in between, link status will be a random variable based on IPG Markov Chain, $\Theta_{n}$. (see Figure \ref{fig:Connectivity}, where we have 6 communication links and $2^{6}$ different connectivity values)

Considering i.i.d. links and running $\Theta_{n}$ for each link, each element of $J$ is calculated for any time instant, showing whether the $k^{th}$ link is up at time $t$. The total number of different states in $J$ is $2^{N}$, where $N$ is the number of links. In this paper, $J$ has 64 different states. State 64 is fully connected mode and state 1 is fully disconnected mode. In order to analyze the mean square stability of the networked system, we need the transition probability matrix (TPM) of the connectivity values. The probability for each connectivity is generated from individual link connection probabilities. In other words, the Markov Chain for connectivity is a random function of the IPG Markov Chain, $\Theta_{n}$. 

%%%%%%%%%%%%%%%%%%%%%%%%%%%%%%%%%%%%%%
\section{Vehicle Modeling and Control}
%%%%%%%%%%%%%%%%%%%%%%%%%%%%%%%%%%%%%%
\noindent In this section, we study how far the stability of the considered platooning application can be analyzed mathematically. We describe a system model, summarize the standard Markov Jump Linear Systems (MJLS) method and discuss its limitation due to the exponential growth of state space size. We propose an alternative problem stetting whose analysis is tractable but which comes at the price of a slow convergence rate.

%%%%%%%%%%%%%%%%%%%%%%%%
\subsection{System Model}
\label{sec:system_model}
%%%%%%%%%%%%%%%%%%%%%%%%
\noindent We consider a double integrator model governed by the dynamics
\begin{equation}
\label{eqn:agent_dynamics}
    \begin{split}
      \dot{x} &= v \\
      \dot{v} &= u
    \end{split}
\end{equation}
with a first-order-hold (FOH) discretization with a sampling time $T$, the discrete-time dynamics can then be represented as
\begin{equation} \label{eqn:agent_dynamics_ZOH}
\begin{split}
  x_{k+1} &= x_{k} + T v_k + \frac{T^2}{2}u_k\\
  v_{k+1} &= v_k+T u_k
\end{split}
\end{equation}

We consider the standard second order consensus protocol (See section 7.4 from \cite{FB-LNS}) with reference position offset $r^{\textnormal{pos}}$ and reference velocity offset $r^{\textnormal{vel}}$ 
\begin{equation} \label{eqn:consensus-protocol}
    u_k=K_p L_k(r^{\textnormal{pos}}_k-x_k)-K_dL_k(r^{\textnormal{vel}}_k-v_k),
\end{equation}
where $L_k$ is the Laplacian matrix \cite{FB-LNS} at time $k$ and $K_p$, $K_d$ are control gains.
With state variable $z_k=[x_k \; v_k]^T$ and $r_k=[r^{\textnormal{pos}}_k \;
r^{\textnormal{vel}}_k]^T$, the closed loop dynamics can then be represented as
\begin{equation} \label{eqn:closed-loop MJLS}
z_{k+1}=A_k z_k+B_k r_k
\end{equation}
where,
\begin{equation}
\begin{split}
A_k&=\begin{bmatrix}
I-\frac{K_pT^2}{2}L_k & T I-\frac{K_dT^2}{2}L_k\\
-K_pTL_k & T I-K_dTL_k\\
\end{bmatrix}\\
B_k&=
\begin{bmatrix}
\frac{K_pT^2}{2}L_k & \frac{K_dT^2}{2}L_k\\
K_pTL_k & K_dTL_k\\
\end{bmatrix}
\end{split}
\end{equation}

In this paper, the reference position or the desired intervehicle distance varies in multiple of 25m within the range [25m 175m]. The lead vehicle in the platoon $(V_0)$  represents the speed reference for the following platoon members. The acceleration or deceleration of the leader can be viewed as disturbances in a platoon. In our studies simulation step is $100 ms$ which is the ideal communication periodicity.

\begin{figure}
\centering
\includegraphics[width=1\columnwidth]{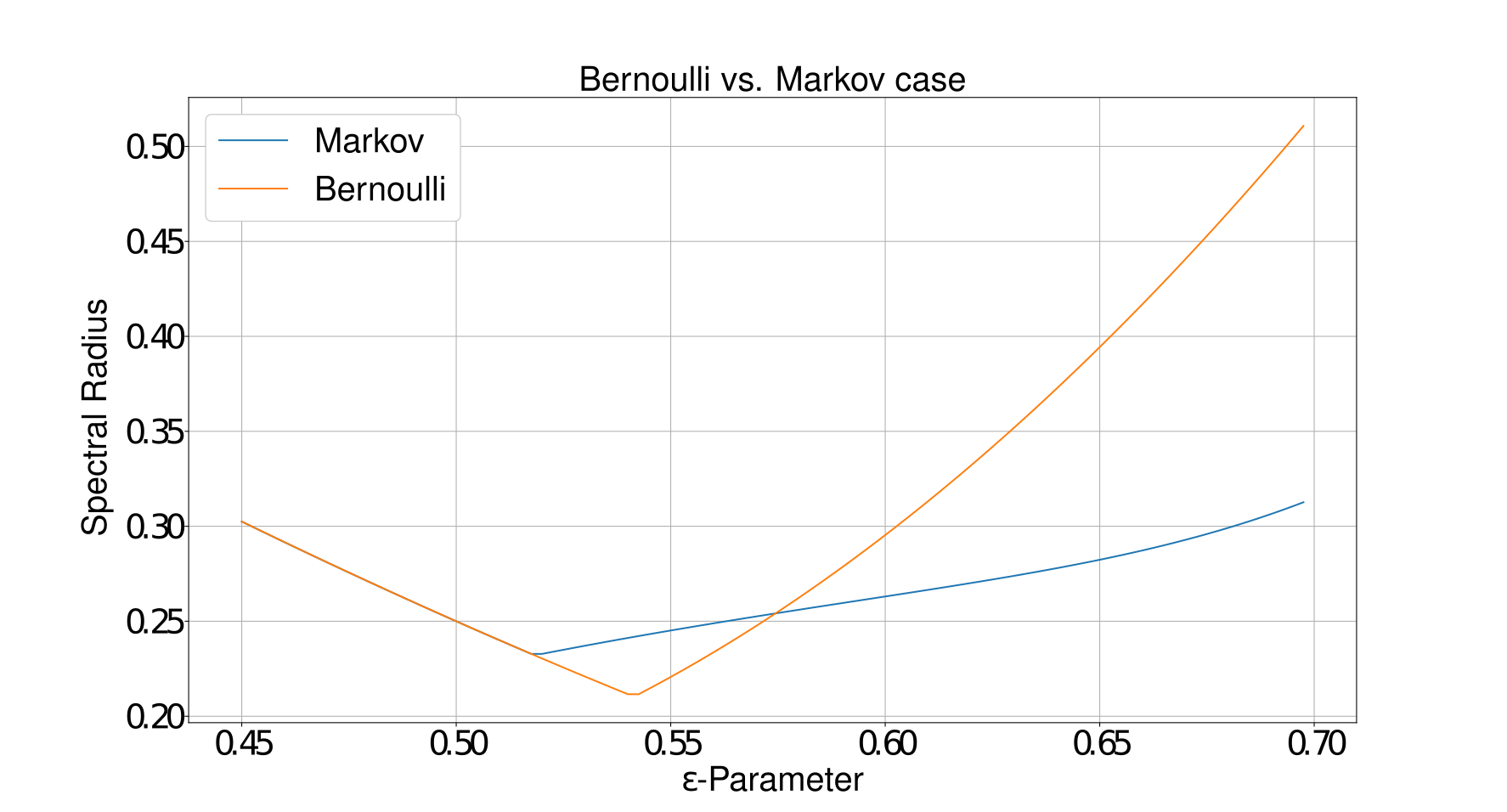}
\caption{Spectral radius over the $\varepsilon$-parameter of Section \ref{section:BernoulliMarkov} for both Bernoulli and Markov system. The example shows that, in general, the Bernoulli system is not more conservative than the MJLS and vice versa.}
\label{Figure:Counterexample}
\end{figure}

%%%%%%%%%%%%%%%%%%%%%%%%%%%%%%%%%%%%%%%%%%%%%%%%%%%%%%%%%%%%%
\subsection{Known Markov Jump Linear Systems (MJLS) Analysis}
\noindent We assume that the Laplacian matrix $L_k$ is a random process described by a Markov chain and sampled for each $k$ from a finite set $\mathcal{L}=\{\mathcal{L}_1,\mathcal{L}_2,\cdots ,\mathcal{L}_N \}$. Since the matrices $A_k$ and $B_k$ are functions of $L_k$, these are random process as well sampled from appropriately defined sets $\mathcal{A}=\{ \mathcal{A}_1,\mathcal{A}_2,\cdots ,\mathcal{A}_N \}$ and $\mathcal{B}=\{ \mathcal{B}_1,\mathcal{B}_2,\cdots ,\mathcal{B}_N \}$. 
Let $P\in \mathbf{R}^{N \times N}$ denote the TPM for the underlying Markov chain such that for all $k$,  $p_{ij}=\Prob\{L_{k+1}=\mathcal{L}_j|L_k=\mathcal{L}_i\}$.
Dynamics \eqref{eqn:closed-loop MJLS} is then a Markovian Jump Linear System (MJLS) \cite{Costa_Book}.

For the MJLS \eqref{eqn:closed-loop MJLS}, define a matrix $\mathcal{S}=(P' \kronecker I)\textnormal{blkdiag}(\mathcal{A}_1\kronecker \mathcal{A}_1, \mathcal{A}_2 \kronecker \mathcal{A}_2,\cdots ,\mathcal{A}_N \kronecker \mathcal{A}_N).$ It has been shown \cite{Costa_Book} that the dynamics \eqref{eqn:closed-loop MJLS} are mean square stable if and only if the spectral radius of $\mathcal{S}$ is less than 1, in symbols
    \begin{equation}
        \label{eqn:stability_criterion}
        %\varrho(P' \kronecker I)\textnormal{blkdiag}(A_1\kronecker A_1, A_2 \kronecker A_2,\cdots ,A_N \kronecker A_N) <1).
        \varrho(\mathcal{S}) < 1.
    \end{equation}

% \subsubsection{$\mathcal{R}_{2}$}

%%%%%%%%%%%%%%%%%%%%%%%%%%%%%%%%%%%%%%%%%%%%%%%%%%%%%%%%%%%%
\subsection{Dropping Messages to Deal with State Space Size}
%%%%%%%%%%%%%%%%%%%%%%%%%%%%%%%%%%%%%%%%%%%%%%%%%%%%%%%%%%%%
\noindent The size of $\mathcal{S}$ significantly increases when the state space size $N$ or the system matrices $\mathcal{A}_i$ are increased. A direct numerical evaluation of (\ref{eqn:stability_criterion}) thus becomes intractable at some point. In the following, we study a variant $\tilde z_k$ of the given MJLS where only each $n_0$-th transmission is considered and all others are discarded. The motivation behind that variant is to decorrelate state space dependencies with sufficiently large $n_0$. The system thus behaves approximately like randomly choosing states according to the stationary state space probabilities.

With Proposition~\ref{prp:r2} we derive a sufficient condition for mean square stability of that variant.
The result assures that dropping a sufficient number of consecutive messages can be used as a stabilizing strategy. However, the price of dropping messages is that the (asymptotic) convergence rate is  significantly slower. Note that system parameters for stability of the studied variant due to Proposition~\ref{prp:r2} are not necessarily a conservative parameter setting for stability of the original MJLS, cf. Section \ref{section:BernoulliMarkov}. 

\begin{prp}\label{prp:r2}
Let 
    \begin{equation}\label{eqn:closed-loop MJLS-2}z_{k} = A_k z_{k-1} + B_k r_{k-1},\;\,\; k\in \N, z_0 \in \R^n,\end{equation}
evolve according to section \ref{sec:system_model} and assume the underlying Markov chain is irreducible and positive recurrent with stationary transition probabilities $P$.

Let the stationary probability distribution be given by the probability mass function (pmf) $\pi$ and denote by $A_1^{\pi}, A_2^{\pi}, A_3^{\pi}, \dots$ the (random) matrices sampled independently from $\mathcal{A}$ according to $\pi$ (Bernoulli process).

If $\mathcal{R}_2(\pi) := \varrho(\E[A_1^{\pi} \otimes A_1^{\pi}]) < 1,$ then there is $n_0 \in \N$ such that the process $\tilde z_k, \;k\in \N,\; \tilde z_0 = z_0,$ given by
    $$\tilde z_k =  \begin{cases} z_k, \;\,\; \text{ if } k\, \vert\, n_0, \\
    \tilde z_{k-1},\;\,\; \text{else} \end{cases} $$
is mean square stable.
\end{prp}

\begin{rem}
\begin{enumerate}
    %\item The process $\tilde z_k$ can be regarded as Markov process by introducing $(n_0-1)N$ "waiting" states $s_{j,l},$ where $j = 1, \dots, n_0-1,\, l = 1, \dots, N$ with $\Prob(s_{1,l} \rvert \mathcal{L}_l) = 1,\;\Prob(s_{j+1,l}\rvert s_{j,l}) = 1 $ and $\Prob(\mathcal{L}_j \rvert s_{n_0-1, l}) = (P^{n_0})_{lj},\; l = 1, \dots, N,\,j= 1, \dots, n_0-2.$
    
    %\item To ease notation we have changed the indexing in \eqref{eqn:closed-loop MJLS-2} as opposed to \eqref{eqn:closed-loop MJLS}. However, the propositions in this section -depending only on the long-time behavior - hold in both cases.
    \item The process $\tilde z_k$ can be regarded as Markov process by introducing $(n_0-1)N$ "waiting" states $s_{j,l},$ where $j = 1, \dots, n_0-1,\, l = 1, \dots, N$ with $\Prob(L_{k+1} = s_{1,l} \rvert L_{k} = \mathcal{L}_l) = 1,\;\Prob(L_{k+1} = s_{j+1,l}\rvert L_{k} = s_{j,l}) = 1 $ and $\Prob(L_{k+1} = \mathcal{L}_j \rvert L_{k} = s_{n_0-1, l}) = (P^{n_0})_{lj},\; l = 1, \dots, N,\,j= 1, \dots, n_0-2,\; k\in \N.$
\item The proof shows that the result also holds for any $\hat n_0 \ge n_0.$
\end{enumerate}
\end{rem}
%The proof requires the following lemma which follows from standard arguments. 
The proof requires the stability criterion 1) of the following proposition, which is also of interest in its own right.

\begin{prp}\label{lem:stability_criterions}
Let $z_k$ be as in Proposition \ref{prp:r2} with $B_k=0, k\in \N.$ Moreover, consider a more general initial state $z_0 \in L^2(\C^n)$, i.e. $z_0$ is a $\C^n$ valued random variable such that $\E \norm{z_0}^2 < \infty.$ 
The following are equivalent:
\begin{enumerate}
    \item $\limsup_{k\rightarrow \infty}(\E\norm{A_k \cdots A_1 }^2)^{1/k} < 1.$ 
    %\item $\limsup_{k\rightarrow \infty}\norm{\E[z_k z_k^T]}_F^{1/k} < 1, z_0 \in \R^n.$
    \item For all $z_0 \in L^2(\C^n),$ which may depend on any of the Markov chain's random initial states, we have
       \begin{equation}\label{eqn:abs_convergence}\sum_{k=0}^{\infty}\E \norm{z_k}^2< \infty.\end{equation}
        
    %\item $\varrho((P' \otimes I)D) < 1.$
    \item $\varrho(\mathcal{S}) < 1.$
    
    \item There is $\alpha \ge 1$ and $0< \zeta < 1$ such that for all initial vectors $z_0\in L^2(\C^n),$ which may depend on any of the Markov chain's random initial states, we have
   %     $$\E\norm{z_k}^2 \le e^{-k \alpha} \E \norm{z_0}^2 \mbox{ for all } k \ge n_0, \; z_0 \in \R^n.$$
        %$$\E\norm{z_k}^2 \le e^{-k \alpha} \E \norm{z_0}^2 \mbox{ for all } k \ge n_0.$$
        $$E\norm{z_k}^2 \le  \alpha  \zeta^k \cdot\E \norm{z_0}^2 \mbox{ for all } k \in \N.$$
\end{enumerate}
%Here $\norm{\,\cdot\,}_F := \norm{vec(\,\cdot\,)}$ denotes the Frobenius norm.
%Here, for a matrix $M\in \R^{n\times n},$ we denote by $\norm{M}_F := \sqrt{\sum_{i,j=1}^{n} \vert m_{ij} \vert^2}$ the Frobenius norm. %\norm{vec(\,\cdot\,)}$ denotes the Frobenius norm
\end{prp}

\begin{figure}
\centering
\includegraphics[width=1\columnwidth]{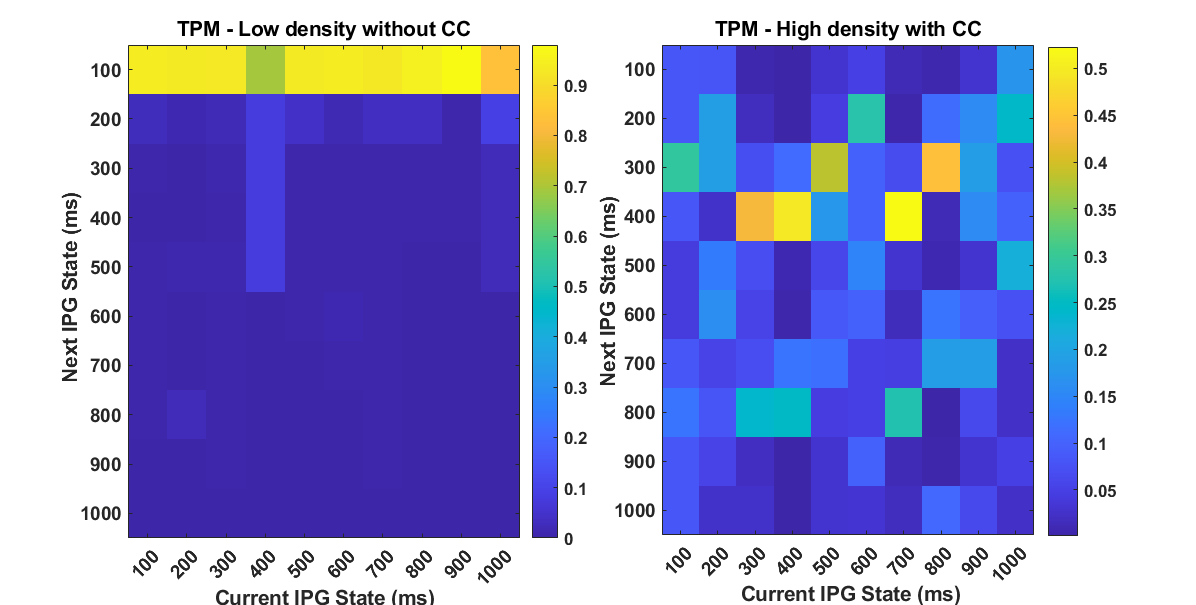}
\caption{Transition probability matrix of IPG Markov Chain for Low density vehicle environment without congestion control in 50 m distance (left side) and transition probability matrix of IPG Markov Chain for High density vehicle environment with congestion control in 500 m distance (right side)}
\label{Figure:TPM}
\end{figure}

\begin{figure*}[h]
\centering
 \includegraphics[width=1\textwidth]{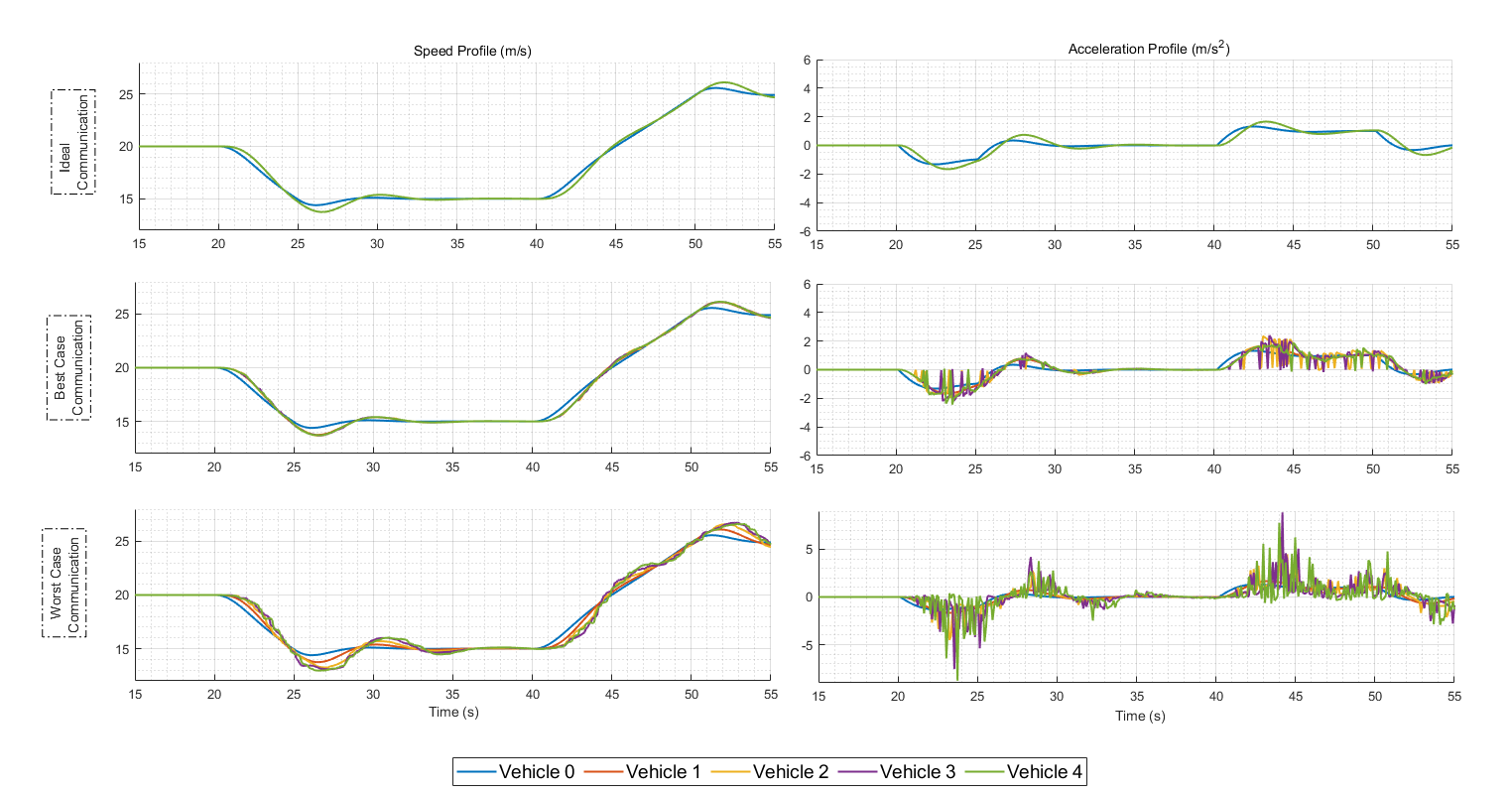}
    \caption{Speed and Acceleration profiles for Five-vehicle test for the ideal communication (No loss), best case communication ( Low vehicular density with Baseline communication and with 25 meters inter vehicle distance), and worst-case communication ( High vehicular density with Congestion Control applied and within 175 meters inter vehicle distance) from top to bottom}
    \label{Figure:profile}
\end{figure*}

\begin{rem}
The expression on the left-hand side of assertion 1) is well-defined as it is independent of the random initial choice of the Markov chain. This follows from arguments similar to the ones elaborated in the proof.
\end{rem}

\begin{proof}[Proof of Proposition  \ref{lem:stability_criterions}]
%The equivalence of 1), 2) and 3) is 
For the equivalence of 2-4) we refer to \cite{Costa_Book}.

It remains to show that 1) implies 2), and that 4) implies 1). 

"$1) \Rightarrow 2)$": 
Denote the random initial state of the Markov chain by $\theta_0.$ From the proofs of  \cite[Proposition 3.24 and 3.23]{Costa_Book}, it suffices to assume that \eqref{eqn:abs_convergence} holds for initial conditions $z_0 \in L^2(\C^n)$ of the following form:
Let $z_0(\omega) := Y_i(\omega) \mbox{ if }\theta_0(\omega) = \mathcal{L}_i, i = 1, \dots, N,$ for random variables $Y_1, \dots, Y_N\in L^2(\C^n)$ which are independent of the Markov states. %Then $z_0 \cdot \mathbf{1}_{\{\theta_0 = \mathcal{L}_i\}} := Y_i.$

This ensures that we can encode any $N$-tuple of positive semidefinte matrices in terms of $\E z_0 \bar z_0',$ as required in \cite{Costa_Book}.

Let $\pi_{0i} = \Pr(\theta_0 = \mathcal{L}_i)$. Then, exploiting independence in the second step below, we have
$$\begin{aligned}
    \E \norm{z_k}^2 &\le \sum_{i=1}^{N}\E[\norm{A_k\cdots A_1}^2 \cdot \norm{z_0}^2 \rvert \theta_0 = \mathcal{L}_i] \cdot  \pi_{0i}\\
    &= \sum_{i=1}^N \E[\norm{A_k\cdots A_1}^2 \rvert \theta_0 = \mathcal{L}_i]\cdot  \E \norm{Y_i}^2\cdot \pi_{0i}\\
    &\le \E\norm{A_k\cdots A_1}^2 \cdot  \max_{i=1}^N \E \norm{Y_i}^2.
\end{aligned}
$$
Assertion 2) now follows from the root and direct comparison test.
%Assertion 2) now follows from fact that the $n$-th root of a nonzero constant converges to one and the finiteness of the state space.

"$4) \Rightarrow 1)$": This is a simple consequence of the fact that we are free to choose the Frobenius norm in assertion 1).
%
%
%The equivalence of 1) and 2) is
%due to the fact that the $n$-th root of a nonzero constant converges to one and the finiteness of the state space.% and due to the Cauchy-Schwarz inequality. 
%
%
%4) $\Rightarrow 2)$: Let $W =(P' \otimes I)D.$ Then the dynamics of
%    $\E z_k z_k^T$ is given by (cf. \cite{Costa_Book})
%     $$\E z_k z_k^T = W^k z_0 z_0^T$$
%     and $\varrho(W) < 1 $ implies 2).
%    
%1) $\Rightarrow 4)$: Cleary, 1) implies mean square stability and it follows 4) \cite{Costa_Book}.
%
\end{proof}

\begin{figure}
\centering
\includegraphics[width=1\columnwidth]{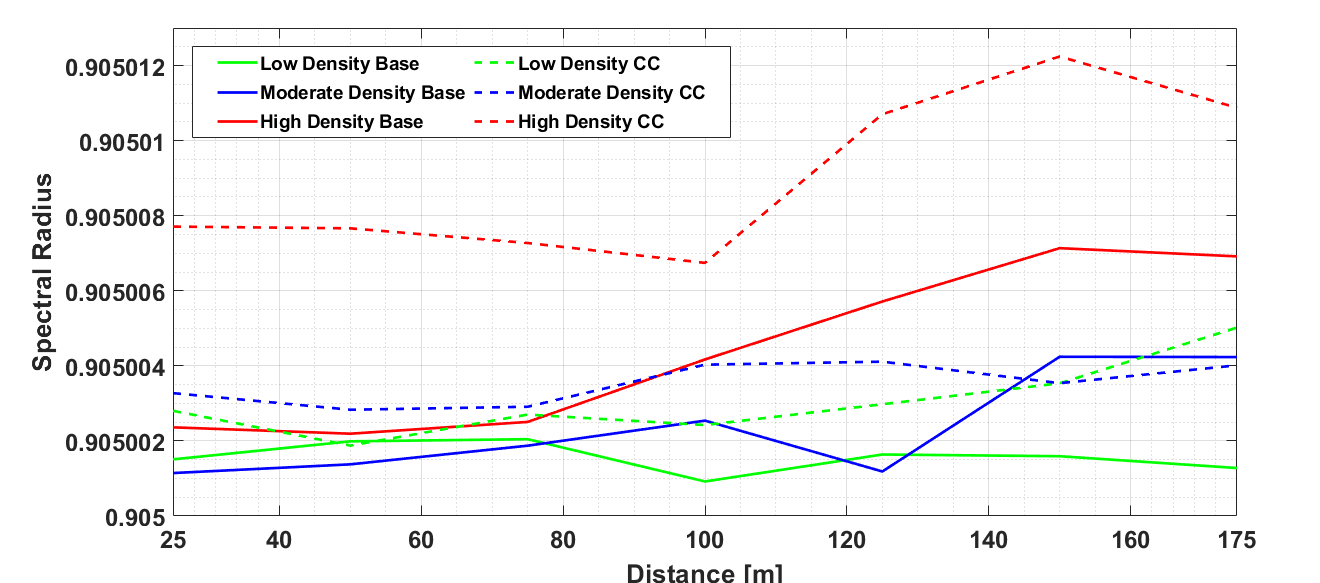}
\caption{Spectral radius for Baseline communication and Congestion Control algorithm for different inter vehicle distances}
\label{Figure:Stability}
\end{figure}

\begin{figure*}[h]
\centering
\includegraphics[width=1\textwidth]{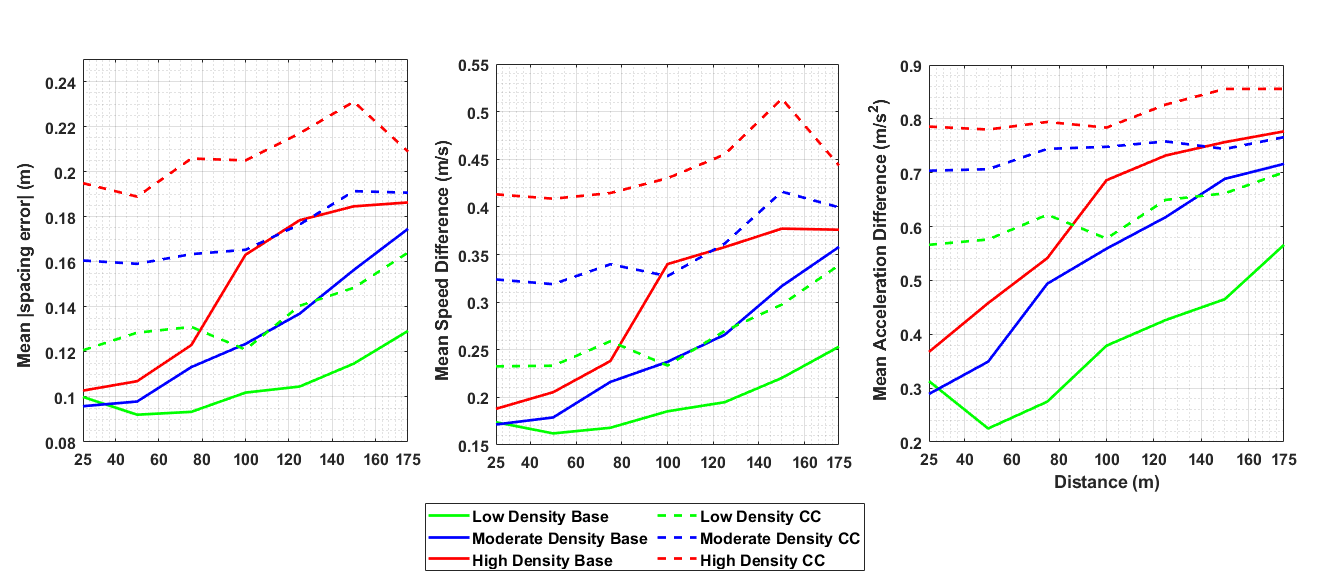}
\caption{Mean of the absolute value of Spacing error, mean of the speed difference, and mean of acceleration difference for Baseline communication and Congestion Control algorithm for different intervehicle distances and different vehicular densities (from left to right)}
\label{Figure:Performance}
\end{figure*}

\begin{proof}[Proof of Proposition \ref{prp:r2} ]
We may assume $B_k = 0$ for all $k\in \N$ (cf. \cite{Costa_Book}). 
Denote by $Q$ the TPM given by the stationary distribution $\pi$, i.e. $Q' = (\pi, \dots, \pi).$
The respective Markov chain is then the Bernoulli process  $w_k = A_k^\pi w_{k-1},\, k\in \N,\, w_0 = z_0.$

Then it is known that 
    $$r:= \varrho((Q\kronecker I)D) =\varrho(\E[A_1^\pi \kronecker A_1^\pi]),$$ where $D$ is the block diagonal matrix $\text{blkdiag}(\mathcal{A}_1 \kronecker \mathcal{A}_1, \dots, \mathcal{A}_N \kronecker \mathcal{A}_N)$
and hence, by hypothesis, $r<1.$

Since the spectral radius is a continuous function from the space of matrices to the real numbers, and since the iterates of the TPM, $P^n$, converge to $Q$ as $n$ approaches infinity, there is $n_0\in \N$ such that $\lvert\varrho((P^{n_0}\kronecker I)D) - r\rvert < 1-r$, thus 
\begin{equation}\label{eqn:markov_iters}\varrho((P^{n_0}\kronecker I)D) < 1.
\end{equation}
Let $\tilde A_k = A_k$ if $ k\, \vert\, n_0,$ and $\tilde A_k = I$ else.
Moreover, the process induced by the system matrices $A_{n_0}, A_{2 n_0}, A_{3 n_0}, \dots$, is Markovian with stationary TPM $P^{n_0}.$ Therefore, %using (\ref{eqn:markov_iters}),
$$\begin{aligned}
&\limsup_{k\rightarrow \infty}(\E \norm{\tilde A_{n_0\cdot k} \cdots \tilde A_{1}}^2)^{1/(n_0\cdot k)} \\
=&\limsup_{k\rightarrow \infty} (\E\norm{A_{n_0 k}\cdot A_{n_0(k-1)} \cdots A_{n_0\cdot 2}\cdot A_{n_0\cdot 1}}^2)^{1/(n_0 k)}\\
=&(\limsup_{k\rightarrow \infty} (\E\norm{A_{n_0 k}\cdot A_{n_0(k-1)} \cdots A_{n_0\cdot 2}\cdot A_{n_0\cdot 1}}^{2})^{1/ k})^{1/n_0}\\
<&\,1, 
%\le& \varrho((P^{n_0}\kronecker I) D)^{1/n_0} <1.
\end{aligned}
$$
by Proposition \ref{lem:stability_criterions} and \eqref{eqn:markov_iters}. Hence, the simple form of $\tilde A_k$ immediately gives
    $$\limsup_{k\rightarrow \infty}(\E \norm{\tilde A_{k} \cdots \tilde A_{1}}^2)^{1/ k} < 1.$$
Since $\tilde z_k$ is a Markov process, this
implies mean square stability of $\tilde z_k$ by Proposition \ref{lem:stability_criterions}, concluding the proof.
%
%
%where $D$ is the block diagonal matrix $blkdiag(A_1 \kronecker A_1, \dots, A_N \kronecker A_N)$
\end{proof}

%%%%%%%%%%%%%%%%%%%%%%%%%%%%%%%%%%%%%%%%%%%%%%%%%%%%%%%%%%%%%%%%%%%%%%%%%%%%%
 \subsection{The MJLS is not Necessarily Stable by Proposition~\ref{prp:r2}}
 \label{section:BernoulliMarkov}
%%%%%%%%%%%%%%%%%%%%%%%%%%%%%%%%%%%%%%%%%%%%%%%%%%%%%%%%%%%%%%%%%%%%%%%%%%%%%
\noindent Proposition \ref{prp:r2} ensures stability of a variant of the initial MJLS if the corresponding Bernoulli system is stable. However, we may not conclude that the initial MJLS is stable as well.  Let us consider a simple example where the system dynamics is given by \eqref{eqn:closed-loop MJLS} with $B_k = 0$ and $A_k = (I-\frac{1}{n}\ones{}\,\ones') (I - \varepsilon L_k),$ where $\ones$ denotes the vector of all ones. Thus $A_k$ defines a first order consensus protocol. We use the configuration for moderate vehicular density and communication with congestion control for 100m distance here. 

Figure \ref{Figure:Counterexample} shows spectral radius curves of both the Bernoulli and the Markov system over the system parameter $\varepsilon$. Since the curves intersect, there are system parameters where the Markov system behaves more conservatively than the Bernoulli system and vice versa. Thus, we cannot simply use parameters that stabilize the Bernoulli System and hope for the stability of the MJLS.

%%%%%%%%%%%%%%%%%
\section{Results}
%%%%%%%%%%%%%%%%%
\noindent In this paper we used two communication modes: firstly, a baseline communication where messages are broadcast every 100 ms (labeled as "Base" in figures) and secondly, under congestion control algorithm which allows broadcast interval in the range of [100ms 600ms] (referred to as CC).

The TPM for two different IPG Markov chains is illustrated in Figure \ref{Figure:TPM}. For the one with low vehicle density and baseline communication mode, regardless of the current IPG state, the most probable next state is 100ms. In other words, regardless of how late we receive the current packet, we will receive the next one in 100 ms. While in the TPM of the case with high vehicle density and CC communication, the next IPG state highly depends on the current IPG state. In this case, the link situation has high importance.

The speed and acceleration profile of a 5 vehicle string for ideal communication (no loss), best case communication (low vehicular density with baseline communication and with 25m inter-vehicle distance), and worst-case communication ( high vehicular density with congestion control applied and with 175m inter-vehicle distance) are shown in Figure \ref{Figure:profile}. In all plots, the leader was supposed to reduce its speed from $20\ m/s$ to $15\ m/s$ and accelerate to reach $25\ m/s$. In the worst case, the PD controller was able to maintain the desired velocity, but a sudden change in acceleration appeared which will result in passenger inconvenience.

Moreover, some properties (e.g., stability and scalability) are not only related to decentralized controllers, but also depend on the information flow topology. By embedding erasure links in our algorithms as a Markov chain, stability properties become dependent on the erasure links. Consequently, it will allow us to analyze the impact of erasure links on platoon control performance. Figure \ref{Figure:Stability} shows the spectral radius of described MJLS network for both baseline communication and CC algorithm in 25m - 175 m range. As we increase the vehicular density, spectral radius increases (red lines are above blue and green lines). Also, we see a higher spectral radius value for the CC algorithm rather than baseline communication in Figure \ref{Figure:Stability} (dashed lines are above regular lines). Note that if the communication rate is higher, then the probability of receiving a packet will be larger even if the Packet Error Rate (PER) is high. Furthermore, the spectral radius has an increasing trend with the distance between connected vehicles. For all cases, the network is stable, and the spectral radius value is below one.

We defined the error in platooning as the absolute value of the difference between the actual distance gap and desired distance gap in meters. Also, the difference between maximum speed and minimum speed and maximum acceleration and minimum acceleration considering all string members at each time step is a good measure for traffic flow and network performance. We defined these metrics as speed difference and acceleration difference.

Figure \ref{Figure:Performance} shows the mean of the absolute value of spacing error, mean of the speed difference, and mean of acceleration difference for Baseline Communication and CC algorithm in the 25m - 175 m range for intervehicle distance. By increasing the vehicle density, the performance degrades as expected (red lines are above blue and green lines in all plots, also dashed lines are above solid lines).
%%%%%%%%%%%%%%%%%%%%
\section{Conclusion}
%%%%%%%%%%%%%%%%%%%%
\noindent In this paper, we have addressed the problem of communication erasure links on the control performance of CAVs under the consensus framework for Platooning application. Markov chain models are used to represent switching network topologies resulting from link erasures. The impact of communication erasure links on vehicle platoon formation, robustness, and stability is analyzed and illustrated. Our findings show that platooning with C-V2X is stable using a simple PD controller and verified by various numerical simulations. The numerical simulations demonstrate that communication rate and loss ratio have a significant impact on vehicle platoon performance. For safety and efficiency, a better controller and communication model, i.e. Model Predictive Control (MPC), and Model-Based Communication (MBC) is more suitable as presented in \cite{Mosh2111:Gaussian}.
%%%%%%%%%%%%%%%%%%%%%%%%%
\section{Acknowledgement}
%%%%%%%%%%%%%%%%%%%%%%%%%
\noindent This research was partially supported by the United State's National Science Foundation under grant numbers CNS-1932037 and CNS-1931981 and partially funded by the German Research Foundation under DFG grant numbers FR-2978/2-2 and WE-2176/14-2.

\balance
\bibliography{refs.bib}{}
\bibliographystyle{unsrt}
\end{document}